\documentclass[runningheads]{llncs}

\usepackage{amsmath,amssymb,amsfonts}
\usepackage{booktabs} 
\usepackage{caption} 
\usepackage{subcaption} 
\usepackage{graphicx}
\usepackage{pgfplots}
\usepackage[utf8]{inputenc}
\usepackage{multicol}
\usepackage{algpseudocode,algorithm,algorithmicx}

\usepackage{graphicx}

\newcommand{\cT}{\mathcal{T}}
\newcommand{\cX}{\mathcal{X}}

\definecolor{blue}{HTML}{1F77B4}
\definecolor{orange}{HTML}{FF7F0E}
\definecolor{green}{HTML}{2CA02C}

\pgfplotsset{compat=1.14}

\begin{document}

\title{DBSCAN of Multi-Slice Clustering for Third-Order Tensors}

\author{Dina Faneva Andriantsiory\inst{1} \and
Joseph Ben Geloun\inst{1} \and
Mustapha Lebbah\inst{2}}

\institute{   LIPN, UMR CNRS 7030  , Sorbonne Paris Nord University, Villetaneuse, France\\
\and
DAVID Lab, University of Versailles, Universit\'e Paris-Saclay, Versailles, France \\
}

\maketitle            

\begin{abstract}
Several methods for triclustering three-dimensional data require as hyperparameters the cluster size set or the number of clusters in each dimension. These methods raise an issue since, for real datasets, those inputs cannot be known without extreme cost. 
Recently introduced, 
the Multi-Slice Clustering (MSC) tackles this issue by using
a threshold parameter to perform the data clustering. The MSC finds signal slices that lie in a lower dimensional subspace of 3rd-order rank-1 tensor datasets. The present work  
addresses an extension
of this algorithm, namely the MSC-DBSCAN, that extracts several slice clusters that lie in  different subspaces, when the 
3rd-order dataset is a sum of $r \ge 1$ rank-1 tensors.
Our algorithm uses the same input as the MSC algorithm and reduces to the same cluster solution for rank-1 tensor dataset. 
\end{abstract}

\section{Introduction}
From an algebraic point of view, an $n$-way or $n$-th order tensor is an element of the tensor product of $n$ vector spaces, each of which has its own coordinate system \cite{kolda2009tensor}. 
The tensor order indicates the number of dimensions of the array.  
As a data structure, consider $m_1$ individuals with $m_2$ features and collect the data for each individual-feature pair at $m_3$ different times. This is an example of a dataset structured in three dimensions, see figure \ref{fig:tensor}. A convenient way to encode such data is given by a tensor of order $3$. Multidimensional data of this kind arises in several contexts such as neuroscience \cite{dataneuroscience}, genomics data \cite{datagenomics1,datagenomics2} and computer vision \cite{datacomputerVision}.  
To learn this data without having a detailed understanding of it, we use unsupervised learning.  Due to their increasing complexity, the mining of higher dimensional data naturally proceeds by identifying subspaces with specific features. Clustering is one of the most popular unsupervised machine learning methods for extracting relevant information such as structural similarity in data. It proceeds by the partitioning of a dataset into meaningful groups or clusters.  The cluster can be defined as a collection of objects that are similar to each other and dissimilar or different from the objects in the other clusters. Thus, several computational methods for clustering multidimensional data, from matrices to higher order tensors, have been introduced in the literature, see for example the review \cite{HRMSC2018SurveyTriclustering}. 
For a detailed account of various available approaches closer to our present work, see \cite{MichaelKathrin2008NTF,kolda2009tensor,HoreTensorDecompositiongene2016,FeiziNIPS2017,DinaTB,TCarson2017,BachLearningSpectralClustering}.

The general structure of clustering algorithms requires, obviously, the data to be treated but also hyperparameters (inputs): the number of clusters or cluster sizes. These  hyperparameters are not easy to set, and more to the point, especially
difficult to set for real data. Moreover, their values influence the quality of the algorithm's output. In \cite{andriantsiory2021multislice}, the authors provide an method that replaces the number of clusters with a measure of similarity within a cluster. This method is called Multi-Slice Clustering (MSC) and it performs on  3rd-order tensors. The strong point of the MSC algorithm  guarantees the quality of the output clustering according to an input threshold
parameter 
(strong similarity within a cluster and strong cluster separation). The MSC method performs well while possesses   
less constraints on hyperparameters. It provides a
good grasp on the cluster quality, and remains competitive compared to other methods for biclustering and triclustering such as 
TFS \cite{FeiziNIPS2017}, Tucker+$k$-means, 
CP+$k$-means \cite{WillSun2019}. 

There is however a limitation of the application domain of the MSC method. Indeed, the latter is designed to find one cluster within rank-1 tensor datasets. 
Such a condition is equally difficult to verify for real data. Therefore, it is essential to extend this method to higher rank tensor datasets to enable multiple cluster learning.

In this paper, we present an extension of the MSC algorithm.  
The new method, named MSC-DBSCAN, is able to identify many clusters within a tensor dataset decomposed as a sum of several rank-1 tensors. 
As the name suggests it, this is a combined method of the MSC 
and the so-called DBSCAN (Density-Based Spatial Clustering of Applications with Noise)  \cite{ester1996densitydbscan}. 
We start with the MSC method which delivers a cluster, then this output goes to the next step of either validation or refinement by the  DBSCAN method. 
The full procedure 
takes the same input 
threshold parameter of  the MSC method. 
If the input data is a rank-1 tensor, then the output is identical to that of the MSC algorithm.  
For $k>1$ rank-1 tensor data, the extension makes a major difference 
compared to the MSC alone. 
In the latter case, the last DBSCAN process will split the
larger cluster in several distinct clusters. 
We experimentally demonstrate the performance of our 
proposed method. 

The paper is organized as follows:  in section \ref{sec:problemFormulation}, we list our notation and detail our model. 
Section \ref{sec:msc} reviews the MSC algorithm and its limitations. Then, in section \ref{sec:mscExtension}, we present the MSC-DBSCAN algorithm after having set up its theoretical approach. In section \ref{sec:experiments},  we apply the new algorithm to  synthetic and real datasets and compare its performances with the MSC algorithm. 

\

\begin{figure}[ht]
	 	\centering
	 	\includegraphics[width=7cm, height=3.8cm ]{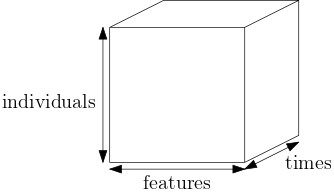}
	 	\caption{The 3-order tensor dataset.}
	 	\label{fig:tensor}
	\end{figure}

\section{Notation and model}
\label{sec:problemFormulation}
\subsection{Notation}
In this paper, we use $\mathcal{T}, \mathcal{X}, \mathcal{Z}$ to represent the dataset tensor, signal tensor, and noise tensor respectively. 
The matrices are represented by the capital letters $A, C, T, V, Z $. We use the Matlab notation for the entries of the tensor. $\|M\|$ and $\|M\|_F$  represent the operator norm and the Frobenius norm of the matrix $M$, respectively. For any matrix, $A$ and $B= \vert A \vert$, the entries of $B$ are the absolute value of the entries of $A$. The lowercase boldface letters $\mathbf{x}, \mathbf{v}, \mathbf{w}, \cdots$ represent vectors and $\|\mathbf{x}\|_2$ is the euclidean norm of the vector $\mathbf{x}$,  the lowercase $x,y, \alpha, \gamma,\lambda$ represent scalars. For any set $J$, $\vert J\vert$ denotes the cardinality of $J$ and $\bar{J}$ denotes the complementary of $J$ in a larger set. For an integer $n>0$, we denote $[n] = \{1,\cdots,n\}$.  The asymptotic notation $a(n) = \mathcal{O}(b(n))$ (res. $a(n) = \Omega(b(n))$) means that, there exists a universal constant $c$ such that for sufficiently large $n$, we have $|a(n)| \le c b(n)$ (resp. $\vert a(n) \vert \ge  c b(n)$). 

\subsection{Model}

We will be mainly interested in 3rd-order tensors. Thus, 
let $\mathcal{T} \in \mathbb{R}^{m_1\times m_2\times m_3}$ be our tensor dataset. 
We decompose it as $\mathcal{T} = \mathcal{X} + \mathcal{Z}$, where $\mathcal{X}$ is called the signal tensor and $\mathcal{Z}$ the noise tensor.

A CANDECOMP/PARAFAC (CP) decomposition of a tensor defines it as a sum of $r$ rank-1 tensors \cite{kolda2009tensor}.
Assuming a particular
CP-decomposition of the signal tensor, we write:
\begin{equation}
\mathcal{T} = \mathcal{X} + \mathcal{Z} = \sum_{i = 1}^{r}\gamma_i\, \mathbf{w}_{i}\otimes \mathbf{u}_{i}\otimes\mathbf{v}_{i} + \mathcal{Z} \label{eq:problem}
\end{equation}
where $\forall i, \gamma_i>0$ stands for the signal strength,  $\mathbf{w}_{i}\in \mathbb{R}^{m_1}, \mathbf{u}_{i}\in\mathbb{R}^{m_2}$ and $\mathbf{v}_{i}\in\mathbb{R}^{m_3}$ are unit vectors, and for a fix $i_0$,  $\mathbf{w}_{i_0}\otimes \mathbf{u}_{i_0}\otimes\mathbf{v}_{i_0}$ is called a rank-1 tensor \cite{kolda2009tensor}.  

Each low-dimensional subspace 
or rank-1 tensor corresponds to one cluster.
The cluster sets from mode-1, mode-2, and mode-3 in the $i-$th subspace are denoted by $J_1^i$, $J_2^i$, and $J_3^i$.
Concerning the noise model, 
we assume that the entries of $\mathcal{Z}$ are independently identically distributed (i.i.d) and have a standard normal distribution.
This is the traditional noise model in an unsupervised learning method for tensor data, see for instance \cite{FeiziNIPS2017,Cai2017statistical}.

For a fixed $i$ (a cluster index), 
the Cartesian product of a pair of sets from $J_1^i, J_2^i$, and $J_3^i$ represents the tensor biclustering.   The collection of the three sets defines a triclustering \cite{FeiziNIPS2017,zhao2005tricluster}.


\section{Multi-Slice Clustering for 3-order tensor: a review}
\label{sec:msc}
In this section, we provide a lightening review of \cite{andriantsiory2021multislice}. 

Equation \eqref{eq:problem} with 
$r=1$ describes our tensor modeling. This means that 
the signal tensor  is of the form 
 $ \mathcal{X} = \gamma \mathbf{w}\otimes \mathbf{u} \otimes \mathbf{v}$. 
Fixing one index of $\mathcal{T}$ and collecting the set of data entries defines a matrix or slice of the tensor.
The MSC method aims at finding, in each dimension of the tensor $\mathcal{T}$, the indices of the matrix slices that are highly similar up to a threshold parameter $\epsilon$. The output of the MSC algorithm is therefore a single 3D cluster.  

 Fixing an index $i$ in the mode-1 of $\cT$, we obtain a matrix $T_i\in \mathbb{R}^{m_2\times m_3}$ defined by (Matlab notation) 
\begin{equation}
    T_i = \mathcal{T}(i,:,:) = \mathcal{X}(i,:,:) + \mathcal{Z}(i,:,:) = X_i + Z_i
\end{equation}
The slice $T_i$ is a sum of the corresponding slice from the signal tensor $X_i $ and the noise tensor $Z_i $. A similar principle is applied to mode-2 or mode-3 of the tensor to collect slices in the remaining modes. 

The MSC learns the information within each slice via spectral analysis. To do so, it starts with the computation of the slice covariance matrix $C_i = T_i^t T_i$. 
 The largest eigenvalue and corresponding eigenvector associated with the matrix $C_i$, 
 covers the most relevant information in each slice indexed by $i$.

Our task is the determination of the element of $J_1$ in mode-1. Recovering $J_2$ and $J_3$ in 
the rest of the modes follows the same idea.
For all slices, we denote by $\underline\lambda_i$ 
	the top eigenvalue and $\tilde{\mathbf{v}}_i$  the top eigenvector of $C_i$ for $i\in [m_1]$. The following matrix contains the dispersion of information of all slices  
	\begin{equation}
	V = \begin{bmatrix}\tilde{\lambda}_1\tilde{\mathbf{v}}_1 & \cdots & \tilde{\lambda}_{m_1}\tilde{\mathbf{v}}_{m_1}	\end{bmatrix}
	\end{equation}
 where we set $\tilde{\lambda}_i$ to $ \underline\lambda_i/\lambda, \forall i\in[m_1]$, and $\lambda =  \max(\underline\lambda_1,\cdots,\underline\lambda_{m_1})$.
 We construct the matrix $C$ defined as $C=|V^tV|$ that now contains the similarity between slices. The marginal sum of each row of $C$ is $d_i$, $i\in [m_1]$.  These form a vector $\mathbf{d}$ of marginal sums.  
 The selection of the cluster $J_1$ uses the vector $\mathbf{d}$, and the following theorem guarantees the quality within the output cluster (see theorem III.1 in \cite{andriantsiory2021multislice}). 

\begin{theorem}\label{thm1}	Let $l = \vert J_1\vert$, assume that $\sqrt{\epsilon}\le \frac{1}{m_1-l}$. $\forall i,n\in J_1 $, for $\lambda = \Omega(\mu)$, there is a constant $c_1>0$ such that
	\begin{equation}
	\label{eq:didn}
		\vert d_i - d_{n}\vert \le l\frac{\epsilon}{2} + \sqrt{\log(m_1-l)}
	\end{equation}
	holds with probability at least $1-e(m_1-l)^{-c_1}$, where $\mu=(\sqrt{m_2-1} + \sqrt{m_3})^2$. 
\end{theorem}

The assumption of a rank-1 tensor is viewed as a principal limitation of the application of the 
MSC method. In this section, we propose a radical improvement and extension of the MSC algorithm  to determine more clusters from the input dataset. Therefore, we assume that the dataset is approximated by a sum of $r$ rank-1 tensors (with $r\ge 1$). To illustrate our method, we focus on $r = 2$ but our explanations will easily stand for $r\ge 2$. 
The following statement proves that the MSC is insufficient for 
detecting multiple triclusters. 

\begin{proposition} \label{proposition}
Let $J_1$ and $J_2$ be two disjoint slice clusters that lie in two different subspaces. By simplicity, we assume that $\lambda \approx \lambda_{(1)} \approx \lambda_{(2)}=\Omega(m_1)$ and $\vert J_1\vert = \vert J_2\vert = l$, then the MSC gathers the two clusters as one cluster of slices with high probability.
\end{proposition}
\begin{proof} This is a corollary of theorem \ref{thm1}. We expand the difference $|d_i - d_j|$, with 
$i\in J_1$ and $j \in J_2$ 
into sums of similarity entries that belong  
to $J_1 \cup J_2$ or not.  
Each term will be bounded using theorem \ref{thm1}. 

\end{proof}

Two other theorems in \cite{andriantsiory2021multislice} show the proper separation of the cluster and the rest of the data,
proving also that the MSC performs well for $r=1$. 
Our interest lies  in $r>1$ and will motivate our next study.


\section{MSC-DBSCAN}
\label{sec:mscExtension}

In this section, without loss of generality, we assume that our dataset is a sum of 2 rank-1 tensors.  The generalization of the method to a general tensor follows the same idea.

\subsection{Problem formulation}
Let $\mathcal{T} \in \mathbb{R}^{m_1\times m_2\times m_3}$ be the tensor dataset. 
It decomposes as  $\mathcal{T} = \mathcal{X} + \mathcal{Z}$ where $\mathcal{X}$ is the signal tensor and $\mathcal{Z}$ is the noise tensor. We assume that
\begin{equation}
\mathcal{T} = \mathcal{X} + \mathcal{Z} = \sum_{i = 1}^{r}\gamma_i\, \mathbf{w}_{i}\otimes \mathbf{u}_{i}\otimes\mathbf{v}_{i} + \mathcal{Z} \label{eq:msc_problem}
\end{equation}
where $\forall i, \gamma_i>0$ stands for the signal strength,  $\mathbf{w}_{i}\in \mathbb{R}^{m_1}, \mathbf{u}_{i}\in\mathbb{R}^{m_2}$ and $\mathbf{v}_{i}\in\mathbb{R}^{m_3}$ are unit vectors. 
Thus, the signal $\cX$ 
is a sum of $r$ rank-1 tensors.

The MSC method selects the elements of the cluster from the marginal similarity vector $\mathbf{d}$ of each slice. However, for a dataset with 
CP-decomposition as a sum of more than one rank-1 tensors,   
 two slices of indices $i$ and $j$ ($i\neq j$) which have similarity close to zero may have the same corresponding entries in the vector $\mathbf{d}$.
If the tensor dataset is decomposed as the sum of several rank-1 tensors, the inequality of theorem \ref{thm1} may fail for the determination of a single cluster. 
Thus, there is room for having several clusters
within the detected MSC cluster. 

We illustrate the occurrence of two clusters in the following:
\begin{itemize}
    \item If the two clusters are well separated, there is a gap between the marginal similarity values within the vector $\mathbf{d}$. See figure \ref{fig:J1capJ2}. This case may happen
    in two situations:
    (1) the two clusters have same size, and the top eigenvalues corresponding to the two clusters have a significant difference; (2) the two clusters have the same eigenvalue but the cluster sizes are very different ($ \vert J_2\vert \gg \vert J_1\vert$). In this case, we can run iteratively the MSC method to detect 
    $J_1$ and $J_2$. 
    Indeed,   running first MSC detects the elements of $J_1$ and, then running again the MSC on the remaining dataset will enable to detect the second cluster $J_2$ (we can keep up on the rest of the data).
   
    \item  If there is no significant gap between the two clusters, the MSC method gathers the two clusters in one which is then output. This is because the necessary condition of the equation \eqref{eq:didn} in theorem \ref{thm1} is verified for all elements in $J_1\cup J_2$. We illustrate this case in figure \ref{fig:J1UJ2}. This case happens when: 
    (1) the top eigenvalues associated with the covariance matrices of the 
    cluster slices are equal and the two cluster sets have the same cardinality; 
    (2) another possibility is when 
    there are similar entries (corresponding to  two or more slices) within the vector $\mathbf{d}$. 
    In this case, an iteration of the MSC method will not separate the elements of the clusters. 
    The proof of this statement is presented in proposition \ref{proposition}.
\end{itemize}

	\begin{figure}[htbp]
	 	\centering
	 	\includegraphics[width=9cm, height=4.5cm ]{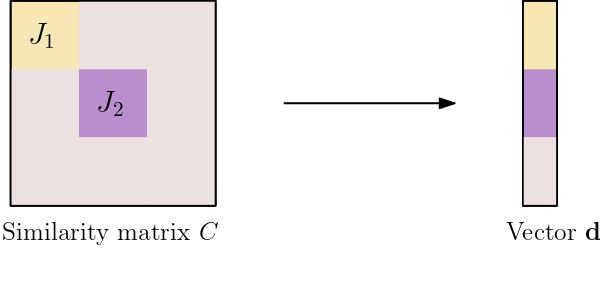}
	 	\caption{Similarity matrix $C$ for a tensor with two disjoint clusters and the corresponding vector $\mathbf{d}$.}
	 	\label{fig:two_disjoint_clusters}
	 \end{figure}

	\begin{figure}[htbp]
 		 \begin{subfigure}[b]{0.4\textwidth}
	 		\centering
	 	\includegraphics[width=5.4cm, height=4.3cm ]{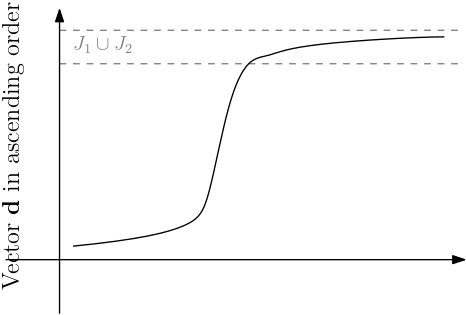}
				\caption{$J_1$ and $J_2$ have similar entries with $\mathbf{d}$.}
			 			\label{fig:J1UJ2}
	 	\end{subfigure}
		\hspace{1.5em}
	 	 \begin{subfigure}[b]{0.4\textwidth}
	 	\centering
	 	\includegraphics[width=5.4cm, height=4.5cm ]{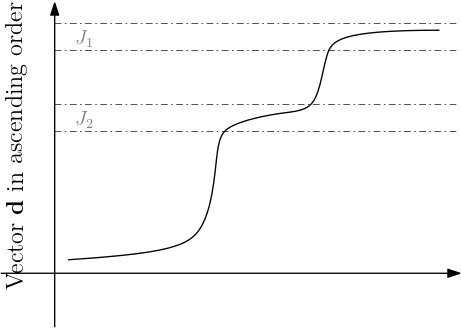}
		\caption{$J_1$ and $J_2$ are well separated within $\mathbf{d}$.}
				\label{fig:J1capJ2}
	 	\end{subfigure}
	 	\caption{Plot of the entries of the vector $\mathbf{d}$ in ascending order.}
	 	\label{fig:vec_d}
	 \end{figure}

We propose  the MSC-DBSCAN method to solve this problem and to distinguish these clusters.  

\subsection{MSC-DBSCAN method}
As its name  suggests it,  
the MSC-DBSCAN method is a sequence of two methods. It starts with the MSC, then it follows a  verification/validation/refinement of the cluster output with the DBSCAN 
analysis (Density-Based Spatial Clustering of Applications with Noise)
\cite{ester1996densitydbscan}. 
The MSC method, based on the hyperparameter $\epsilon$,  first delivers 3 slice clusters, one for each mode. As a second move,  the DBSCAN analysis refines the similarity relation
 between the slices of the selected cluster. 
If the output is exactly one cluster (cannot be further clustered), then the MSC-DBSCAN will confirm it. However, if the output of the MSC may be further decomposed in  different clusters, the DBSCAN will proceed with the splitting of the initial ouput result.
 We must keep in mind that the DBSCAN needs two hyperparameters to detect the cluster in the input: the neighborhood radius and the minimal number of neighborhoods, denoted by $\varepsilon$ and $Minpts$, respectively. Thus, one realizes that the second part of the procedure has all necessary information withdrawn the MSC method.

The hyperparameter $\varepsilon$ is obtained from the hyperparameter of the MSC algorithm $\epsilon$, and the cluster size of the output $l$. The neigborhood radius is 
\begin{equation}
    \varepsilon =  \sqrt{\frac{l\epsilon}{2} +  \sqrt{\log(m_1 -l)} }
\end{equation}
    and we take the minimum neighborhoods as $Minpts = 2$. We use the euclidean distance to measure the dissimilarity in the splitting algorithm. The following theorem justifies the performance of the method.

\begin{theorem}
We denote by $J$ the output clustering output from the MSC algorithm. Let $J_1$ and $J_2$ the two disjoint clusters lie in two different subspaces such that $J=J_1\cup J_2$. Then, 
\begin{itemize}
\item for $i,j \in J_1$ (or belonging to $J_2$), we have
\begin{equation}
d(c_{i.},c_{j.}) \le  \sqrt{\frac{l\epsilon}{2} +  \sqrt{\log(m_1 -l)} },
\end{equation}
\item for $i\in J_1$ and $j\in J_2$, we have,
\begin{equation}
d(c_{i.},c_{j.}) \ge  \sqrt{\frac{l\epsilon}{2} +  \sqrt{\log(m_1 -l)} }.
\end{equation}
\end{itemize}
with a high probability when $m_1$ becoming  large, where $c_{i.}$ represents the $i$-th column of matrix $C$.
\end{theorem}

\begin{proof}
By definition of the vector $\mathbf{d}$ and $i,j \in J_1$ and $|J_1|$=l, we have
\begin{eqnarray}
d(c_{i.}, c_{j.})^2 &= &  \sum_{k\in[m_1]} (c_{ik}-c_{jk})^2 
 \le \sum_{k\in[m_1]} \vert c_{ik}-c_{jk}\vert \crcr
  &\le & \sum_{k\in J_1} \vert c_{ik}-c_{jk}\vert + \sum_{k\in\bar{J}_1} \vert c_{ik}-c_{jk}\vert \crcr
&\le& \frac{l\epsilon}{2} + \sqrt{\log(m_1 -l)}. \nonumber
\end{eqnarray} 
We have the first inequality because $-1\le c_{ik}-c_{jk} \le 1$ for all $k\in[m_1]$. The last inequality is held if by the clustering condition of the MSC method 
(theorems \ref{thm1} and theorem III.2 of \cite{andriantsiory2021multislice}).

For $i\in J_1, j \in J_2$, with the hypothesis $\mathbf{v}_i \perp \mathbf{v}_j$ by the definition of the rank decomposition and $\vert J_1\vert =\vert J_2\vert =l $, we have
\begin{eqnarray}
d(c_{i.},c_{j.})^2 &= &\sum_{k\in[m_1]} (c_{ik}-c_{jk})^2  \crcr
&=& \sum_{k\in[m_1]} c_{ik}^2 + c_{jk}^2 - 2 c_{ik} c_{jk}\crcr
&\ge & \sum_{k\in J_1}( c_{ik}^2 + c_{jk}^2 - 2 c_{ik} c_{jk}) +\sum_{k\in J_2} (c_{ik}^2 + c_{jk}^2 - 2 c_{ik} c_{jk}) \crcr
&& + \sum_{k\in [m_1]\backslash \big( J_1 \bigcup J_2 \big) }  (c_{ik}-c_{jk})^2    \label{eq:cjk_cik}
\end{eqnarray}
According to the hypothesis $\mathbf{v}_i \perp \mathbf{v}_j$, for $j\in J_2$ and $k\in J_1$ we have $c_{jk} =0$. In the same vein, for $i\in J_1 $ and $k\in J_2$, we have $c_{ik} = 0$. Using the $\epsilon$-similarity (\cite{andriantsiory2021multislice}, definition III.2), the equation \eqref{eq:cjk_cik} becomes:
\begin{eqnarray}
d(c_{i.},c_{j.})^2 \ge \sum_{k\in J_1} c_{ik}^2  +\sum_{k\in J_2}  c_{jk}^2 
\ge  2l^2(1-\frac{\epsilon}{2})^2
\end{eqnarray}
Hence, 
\begin{equation}
d(c_{i.},c_{j.}) \ge  \sqrt{2}l(1-\frac{\epsilon}{2}) \ge  \sqrt{\frac{l\epsilon}{2} +  \sqrt{\log(m_1 -l)} }
\end{equation}
which is the claim. 
\end{proof}

The detailed steps of our method is shown in algorithm \ref{algo:msc_extension}.
In this algorithm, the upper-script $j\in\{1,2,3\}$, appearing for instance in $J^{(j)}$, represents the studied mode of the tensor. The under-script $i$ of $J^{(j)}_i$ identifies the different clusters in the corresponding mode $j$.

\begin{algorithm}[h] 
		\caption{MSC-DBSCAN}
  \label{algo:msc_extension}
  \begin{algorithmic}[1]
		\Require 3rd-order tensor $\mathcal{T}\in\mathbb{R}^{m_1\times m_2\times m_3}$, threshold parameter $\epsilon$
		\Ensure $(J_i^{(1)}, J_i^{(2)}, J_i^{(3)} )_i$
		\For{$j$ in $\{1,2,3\}$}
		    \State Initialize the matrix $M$
			   \State Initialize $ \lambda_0\gets 0$
			    \For{ $i$ in $\{1,2,\cdots,m_j\}$}
				    \State Compute : $C_i\gets T_i^t T_i $ 
				    \State Compute the spectral decomposition of $C_i$
				    \State Compute $M(:,i) \gets \lambda_i * \mathbf{v}_i$ 
				    
				     \If {$\lambda_i >\lambda_0$}
  	                 \State $\lambda_0 \gets \lambda_i$ 
  	                \EndIf
			    \EndFor
     \State Compute : $V \gets M / \lambda_0$
     \State Compute : $C \gets \vert V^t V\vert $ 
     \State Compute : $\mathbf{d}$ the vector marginal sum of $C$ and sort it 
     \State Initialization of $J^{(j)}$ using the maximum gap in $\mathbf{d}$ sorted 
     \State Compute : $l\gets \vert J^{(j)}\vert$ 
     
    \While {not convergence of the elements of $J^{(j)}$ (theorem \ref{thm1})}
         \State Update the element of $J^{(j)}$ (excluding $i$ s.t. $d_i$ is the smallest value that violates theorem \ref{thm1}) 
        \State Compute $l$ 
    \EndWhile
   \State  Build the similarity of all indices in $J^{(j)}$ from the matrix $C$ 
   \State Apply DBSCAN algorithm with hyper-parameter $\varepsilon = \sqrt{\frac{l\epsilon}{2} +  \sqrt{\log(m_1 -l)} }$ and $ Minpts = 2$ in order to get the set $(J^{(j)}_i)_i$.
		\EndFor
  \end{algorithmic}		
	\end{algorithm}

\section{Experiments}
\label{sec:experiments}

To evaluate the performance of the present clustering method, we first use two indices suitable for synthetic data. First, we discuss the ARI   \cite{ARI}, and  then the cluster quality with respect to the similarity will tested through the root mean square error (RMSE)  \cite{ChaiRMSEMAE2014}. We start the experiments with the synthetical data, and  then we process a real dataset.

\

\ 

\noindent{\bf Synthetical results --}
Here, we generate a synthetic dataset as a sum of 2 rank-1 tensors plus the noise tensor. We run the MSC and MSC-DBSCAN algorithms on our data, detect the different clusters, and compare their  quality. 

The data is built as in equation \eqref{eq:msc_problem}. We consider a tensor with $r=2$ rank-1 tensors. Each factor matrix contains 2 columns vectors which are orthogonal.  Since the MSC cluster selection is independent in each mode, the following explanation focuses only on the first mode of the tensor. The same operations in the remaining modes are easily
implemented and analogous.

We generate tensor datasets with $m_1=m_2=m_3=50$, and 
introduce two clusters lying in two different subspaces within the data.
In mode-1, the two cluster sets of the slices are denoted by  $J_1$ and $J_2$; they obey  $J_1\bigcap J_2=\emptyset$. Based on the similarity measure, the elements of $J_1$ are highly dissimilar to the elements of $J_2$. The two clusters have the same cardinality $\vert J_1 \vert =\vert J_2\vert =10$. The signal tensor is constructed as the following way:

\begin{eqnarray}
\mathbf{u}_1(i) &=& \left \{
\begin{array}{rcl}
\frac{1}{\sqrt{\vert J_1\vert }}\qquad \text{ if } i\in J_1,\\
0 \qquad \text{otherwise}
\end{array}
\right. \crcr
\mathbf{u}_2(i) &=& \left \{
\begin{array}{rcl}
\frac{1}{\sqrt{\vert J_2\vert }}\qquad \text{ if } i\in J_2,\\
0 \qquad \text{otherwise}
\end{array}
\right.
\label{eq:data}
\end{eqnarray}

The same cluster construction
of equation \eqref{eq:data} 
extends to mode-2 and mode-3 of the tensor. The entries of the tensor noise are standard Gaussian random variable independent and identically distributed.

First, we evaluate the performance of the MSC-DBSCAN  by varying the signal strength value $\gamma_i$, $i=1,2$, for the rank-1 tensors. We assume that the 2 rank-1 tensors have the same weight $\gamma_i = \gamma$ and let $\gamma$ vary from 50 to 100. We run the algorithm 10 times for each value of $\gamma$. The two cluster sets of each mode have the same cardinality.  The value of the $\epsilon$ is chosen to be 0.001 during the experiment in order to satisfy the hypothesis of the main MSC  theorem. For each iteration, we rate the quality of the output by computing its ARI.  
The results are displayed in 
figure \ref{fig:ARI_Extension}.

\begin{figure}[htbp]	
\centering	 	
\includegraphics[width=10cm, height=6.5cm]{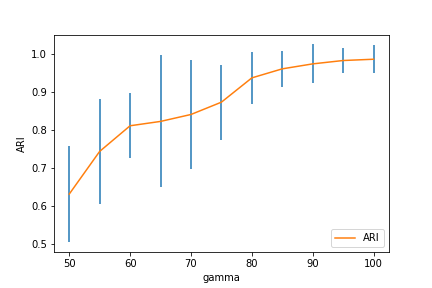}
\caption{The ARI of the output of MSC-DBSCAN for $\gamma$ ranges from 50 to 100.} 	 	
\label{fig:ARI_Extension}
\end{figure}

The curve in figure \ref{fig:ARI_Extension} represents the mean and the vertical blue line represents the standard deviation of the ARI for the 10 experiments. 
As a result, the ARI increases and tends to 1 as the value of $\gamma$ becomes higher. It is stable and very close to the maximum at $\gamma=80$. 
This shows that the MSC-DBSCAN performs well and is able to recover the 2 clusters for increasing values of the signal strength. 

Secondly, on the same synthetic data, we compare the quality  of the MSC and MSC-DBSCAN algorithms with the same value of the hyperparameter $\epsilon$. We set the value of $\gamma$ to $80$. We run the experiment 20 times and regenerate the data for each iteration. The 
RMSE serves as an evaluation of the clustering performance.
This error is computed in the sub-cube (triclustering) generated by the combination of the clusters of the 3 modes.

\begin{figure}[htbp]	
\centering	 	
\includegraphics[width=10cm, height=6.5cm]{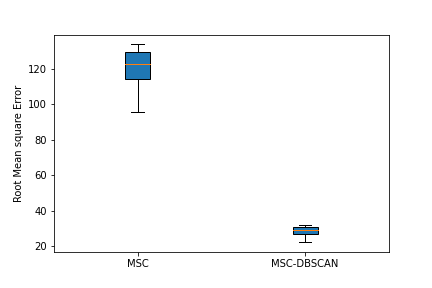}	
\caption{Represents the root mean square error of the sub-cube(s) generated from the output of MSC  and  MSC-DBSCAN algorithms.} 	 	
\label{fig:MSC_extension}
\end{figure}

Figure \ref{fig:MSC_extension} shows that the RMSE of the sub-cube given with the MSC algorithm is very high compared to the RMSE of the sub-tensors delivered by the MSC-DBSCAN algorithm. 
The reason of this difference has the following explanation. There are 2 well-separated clusters in each mode. The MSC algorithm combines the elements of these two clusters into one cluster because they have almost the same value within the vector $\mathbf{d}$. On the other hand, the MSC-DBSCAN has split the output of the MSC method into two clusters. 
This obviously reduces the RMSE (clusters of smaller size generate high similarity). 
The MSC-DBSCAN method enhances
the clustering quality in the case of multiple clusters. 

Although, we have not presented the results here, 
we note the following fact: if the clusters are well separated (a significant difference is noticed in the vector marginal sum $\mathbf{d}$), then the outputs of the MSC and the MSC-DBSCAN are the same.

\

\noindent{\bf Real  data: preliminary results --}
We run the MSC-DBSCAN on the same real data \cite{andriantsiory2021multislice} set used in the MSC. The output of
the MSC-DBSCAN  delivers only one cluster lying in one subspace.  This cluster  is identical to the output of the MSC algorithm. Therefore,
this confirms that such real dataset contains
only a single cluster. 
Due to the lack of real datasets, we have not  perform other experiments on the present algorithm. 
This must be done in the future. 

\

The MSC-DBSCAN  code is available at this link: 

\url{https://github.com/ANDRIANTSIORY/MSC-DBSCAN}.

\section{Conclusion}
 The MSC method and performance rest on a particular type of input: 
 a rank-1 tensor. 
 If the tensor can be decomposed as a sum of a rank-1 tensors, the MSC algorithm is of limited efficiency and even relevance.
 This situation may very well occur for real dataset. 
 We have proposed an  extension and improvement of the MSC method able to address that generic situation: the MSC-DBSCAN. The MSC-DBSCAN provides the same output cluster as MSC for a rank-1 tensor. In addition, it is able to extract several clusters in   3rd-order datasets,  sum of rank-1 tensors. 
 Statistical theorems ensure the method reliability.  
 We have performed conclusive experimental tests and shown the effectiveness of the new algorithm compared to MSC method.

\bibliographystyle{splncs04}
\bibliography{myfile}
\end{document}